\documentclass[letterpaper, 10 pt, conference]{ieeeconf}  

\IEEEoverridecommandlockouts                              

\usepackage{amsmath,amssymb,amsfonts}
\usepackage{graphicx}
\usepackage{makecell}
\usepackage{yhmath}
\usepackage{multirow}
\usepackage{subfigure}
\usepackage{color} 
\usepackage{cite}
\title{\LARGE \bf
An Efficient Method for Extracting the Shortest Path from the Dubins Set for Short Distances Between Initial and Final Positions}

\author{Xuanhao Huang and Chao-Bo Yan$^*$
\thanks{The authors are with the State Key Laboratory for Manufacturing Systems Engineering, and also with the School of Automation Science and Engineering, Faculty of Electronic and Information Engineering, Xi'an Jiaotong University, Xi'an, Shaanxi
710049, China. {\tt\scriptsize xhhuang@stu.xjtu.edu.cn; chaoboyan@mail.xjtu.edu.cn}}
}

\graphicspath{{figures/}}
\newtheorem{theorem}{Theorem}
\newtheorem{proposition}{Proposition}
\begin{document}

\maketitle
\thispagestyle{empty}
\pagestyle{empty}

\begin{abstract}
Path planning is crucial for the efficient operation of Autonomous Mobile Robots (AMRs) in factory environments. Many existing algorithms rely on Dubins paths, which have been adapted for various applications. However, an efficient method for directly determining the shortest Dubins path remains underdeveloped. This paper presents a comprehensive approach to efficiently identify the shortest path within the Dubins set. We classify the initial and final configurations into six equivalency groups based on the quadrants formed by their orientation angle pairs. Paths within each group exhibit shared topological properties, enabling a reduction in the number of candidate cases to analyze. This pre-classification step simplifies the problem and eliminates the need to explicitly compute and compare the lengths of all possible paths. As a result, the proposed method significantly lowers computational complexity. Extensive experiments confirm that our approach consistently outperforms existing methods in terms of computational efficiency.

\end{abstract}

\section{Introduction}
Autonomous Mobile Robots (AMRs) are playing an increasingly important role in the field of logistics transportation. With the global advancement of ``Industry 4.0'', the significance of manufacturing is becoming increasingly pronounced. This trend underscores the need for AMRs to enhance efficiency and reduce operational expenses. 

Path planning is critical for the effective operation of AMRs. In the field of robotic path planning, there are many interpolating curves, such as the Dubins paths, the Bezier curves, and the polynomial curves \cite{gonzalez2015review}. Each of these algorithms has its own strengths. Compared to other paths, the primary advantage of the Dubins paths is that they are proven to be the shortest while satisfying the kinematic constraints of the vehicle \cite{dubins1957curves}. Furthermore, its simple form makes it easy to apply in practice. Based on the aforementioned advantages, the Dubins paths are favored by many researchers and have been modified to achieve different objectives \cite{oliveira2018trajectory, tian2021continuous, schildbach2023continuous, pharpatara20153d, park2022three}.

Dubins' theorem states that any shortest path consists of three segments, each being either a circular arc (C) or a straight line segment (S), forming a sequence of type \textit{CCC} or \textit{CSC}. Each arc \textit{C} corresponds to either a left turn (L) or a right turn (R). Thus, the shortest path is one of six admissible paths, collectively termed the Dubins set: $\{LRL, RLR, LSL, RSR, RSL, LSR\}$, as shown in Fig. \ref{fig:DubinsPathType}. 

\begin{figure}[htbp]
	\centering
	\includegraphics[width=5cm]{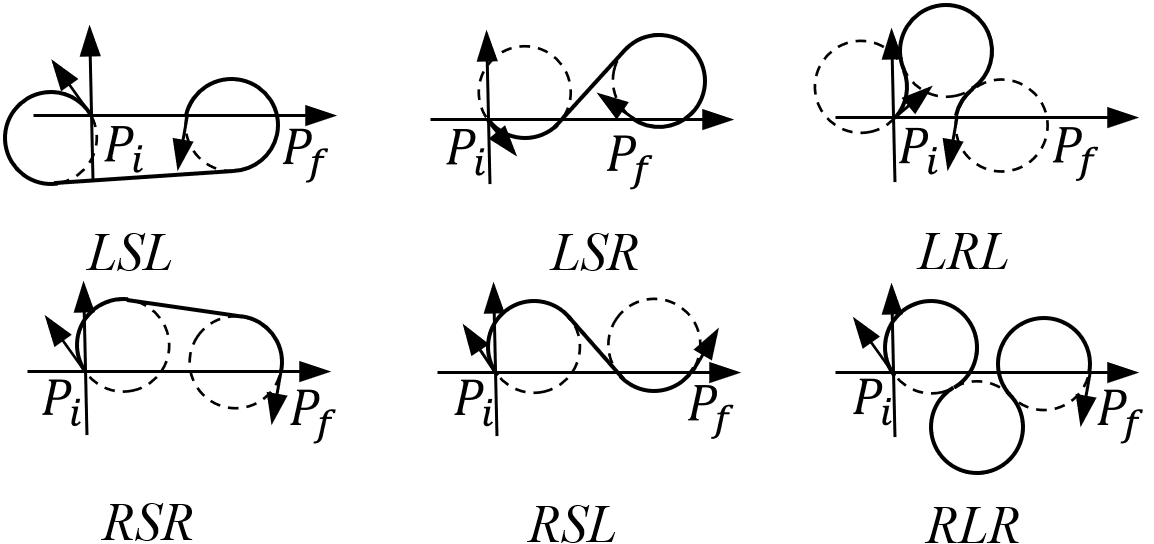}
	\caption{Dubins paths}
	\label{fig:DubinsPathType}
\end{figure}

A natural approach to applying Dubins' result for shortest path calculation is to explicitly compute the lengths of all candidate paths in the Dubins set and then select the shortest one \cite{beard2012small}. Although computing a single Dubins path is very fast, this method entails a significant computational burden if the cumulative cost matters for a massive amount of repeated computations. Despite optimizations for real-time planning, a 3D path planning algorithm takes over 2 seconds to determine a time-optimal path, with 80\% of the time consumed by repeated calculation of Dubins paths \cite{oettershagen2017towards}.

To mitigate the problem of high computational cost, researchers have proposed alternative solutions. Shkel et al. \cite{shkel2001classification} proposed a scheme to efficiently extract the shortest Dubins path when the distance between the initial and final points is large. Cho et al. \cite{cho2006efficient} developed new formulae to explicitly determine the connection points between the circular arc and the straight-line segment in Dubins' \textit{CSC} family of paths. Sadeghi et al. \cite{sadeghi2016efficient} addressed the challenge of computing Dubins paths through three consecutive points by establishing novel geometrical properties of these paths. Additionally, heuristic algorithms are also applied to the construction of Dubins paths \cite{ji20213d, nayak2023heuristics}.

Although computing the length of an individual Dubins path is computationally inexpensive, identifying the shortest path pattern among all possible configurations remains a challenging task. While the shortest Dubins path for relatively large distances has been well studied and effectively resolved, the case involving short distances remains open due to its inherent complexity. To the best of our knowledge, no existing method efficiently addresses this specific scenario.

To develop a complete method capable of handling both cases, we propose a method to efficiently select the shortest path from the Dubins set for case with short distance. Built upon \cite{shkel2001classification}, we classify the initial and final configurations into equivalency groups according to the angle quadrants of the orientation angle pairs and demonstrate how to extract the shortest path for a configuration for each group. The shortest path for different configurations within the same group can be determined by using an orthogonal transformation. Our approach significantly enhances the efficiency of identifying the shortest Dubins path, without explicitly calculating the lengths of all candidates of the Dubins set.

The rest of this paper is arranged as follows. Section \ref{sec:PROBLEM STATEMENT} introduces definition of the short path and properties of the Dubins paths. Section \ref{sec:Complete Calculation of Dubins Curve} presents the extracting method and detailed proofs. Section \ref{sec:Experimental Results} illustrates our main results, indicating the performance of the method. Conclusions and future work are presented in Section \ref{sec:CONCLUTIONS}.

\section{Problem Statement} \label{sec:PROBLEM STATEMENT}

Dubins paths exhibit several fundamental properties that facilitate the analysis. First, the threshold distance between the initial and final positions is introduced, allowing the classification of Dubins paths into the \textit{long path case} and the \textit{short path case}. Subsequently, the lengths of individual path segments are specified. Finally, the Dubins paths are categorized into equivalency groups based on their topological structure. Without loss of generality, we assume a unit radius of the minimal turning circle, i.e., $r = 1$.

The properties of Dubins paths are crucial to the subsequent computation. Following \cite{dubins1957curves}, an admissible path is either (i) a circular arc, followed by a line segment, followed by a circular arc, or (ii) a sequence of three circular arcs, or (iii) a sub-path of a path of type (i) or (ii). Furthermore, \cite{dubins1957curves} states that for a path to be considered optimal, each arc must have the minimum permissible radius.

The lengths of the individual segments of Dubins paths are computed based on the aforementioned properties. To simplify the analysis, a coordinate transformation is applied such that the initial position is relocated to the origin and the final position is set at $(d, 0)$, where $d$ denotes the Euclidean distance between the initial and final positions. The initial and final orientation angles are denoted by $\alpha$ and $\beta$, respectively, as shown in Fig. \ref{fig:AfterTransformation}.

\begin{figure}[htbp]
	\centering
	\includegraphics[width=3.5cm]{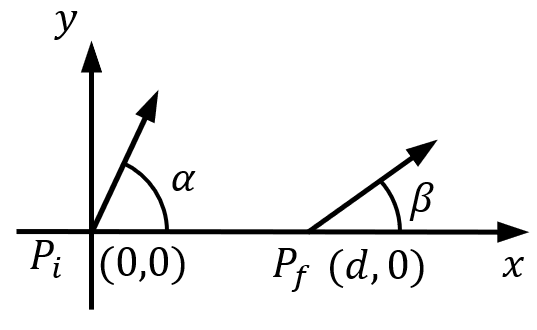}
	\caption{The initial as position $(0, 0)$ with start orientation $\alpha$ and final as position $(d, 0)$ with end orientation $\beta$}
	\label{fig:AfterTransformation}
\end{figure}

\subsection{Short Path Case}

The classification problem can be analyzed more effectively by dividing it into two distinct cases, as illustrated in Fig. \ref{fig:Two cases}. This study specifically focuses on the short-path case. To support this classification, the concept of a threshold distance is introduced.

\begin{figure}[htbp]
	\centering
	\subfigure[The short path case]{
		\includegraphics[width=3.5cm]{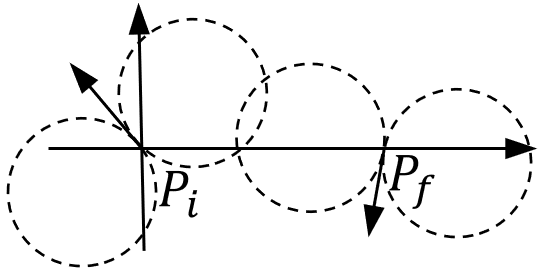}
		\label{fig:ShortPathCase}}
	\subfigure[The long path case]{
		\includegraphics[width=3.5cm]{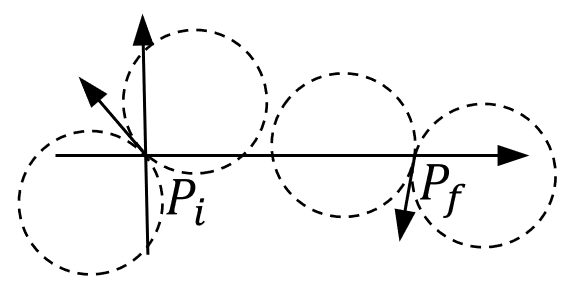}
		\label{fig:LongPathCase}}
	\caption{Two cases of Dubins paths}
	\label{fig:Two cases}
\end{figure}

The initial and final segments of the Dubins paths are circular arcs, denoted $C_{il}$, $C_{ir}$, $C_{fl}$, and $C_{fr}$, where $i$ and $f$ stand for ``initial'' and ``final'', and $l$ and $r$ stand for ``left'' and ``right''. The \textit{short path case} occurs when the union $\{C_{il} \cup C_{ir}\}$ and $\{C_{fl} \cup C_{fr}\}$ have a non-empty intersection, i.e., $\{C_{il} \cup C_{ir}\} \cap \{C_{fl} \cup C_{fr}\} \ne \emptyset$. J. Lim et al. \cite{lim2023circling} proposed that the threshold for the short path case is
\begin{equation}
	\label{eq: d_threshold}
    \begin{aligned}
        d_{t}(\alpha, \beta) = |\sin \alpha| + |\sin \beta| + \sqrt{4 - (\cos \alpha + \cos \beta)^2}.
    \end{aligned}
\end{equation}
If $d$ is less than $d_{t}$, the short path case occurs.

\subsection{Lengths of the Dubins paths}

Dubins paths consist of three types of motion: turning to the left, turning to the right, and moving straight. Consequently, three corresponding operators are defined as following: 
\begin{equation}
	\label{eq: operators}
	\begin{aligned}
		& L_v(x, y, \phi) = (x + \sin(\phi + v) - \sin \phi,\\
		& \qquad \qquad \qquad y - \cos(\phi + v) + \cos \phi, \phi + v), \\
		& R_v(x, y, \phi) = (x - \sin(\phi - v) + \sin \phi,\\
		& \qquad \qquad \qquad y + \cos(\phi - v) - \cos \phi, \phi - v), \\
		& S_v(x, y, \phi) = (x + v\cos \phi, y + v \sin \phi, \phi).
	\end{aligned}
\end{equation}

Herein, the operators represent movement along a segment of length $v$ in the corresponding direction from an arbitrary configuration $(x, y, \phi)$ (where $(x, y)$ represents the position and $\phi$ the orientation).

Using these elementary operators, each segment of a Dubins path can be represented by corresponding analytical expressions. Let $t$, $p$ and $q$ denote the lengths of the initial, middle, and final segments of the path, respectively, with the subscript indicating the type of the path. Let $\mathcal{L}$ represent the length of the corresponding path; for example, $\mathcal{L}_{lsl}$ represents the length of a path of type \textit{LSL} and $t_{lsl}$ refers to the length of its initial segment.

The length of each segment and the total length of \textit{LSL} are derived by applying $L_q(S_p(L_t(0, 0, \alpha))) = (d, 0, \beta)$:
\begin{equation}
	\begin{aligned}
		&t_{lsl} = - \alpha + \arctan \frac{\cos \beta - \cos \alpha}{d + \sin \alpha - \sin \beta}\{\text{mod}~2\pi\}, \\
		&p_{lsl} = \sqrt{2 + d^2 - 2\cos(\alpha - \beta) + 2d(\sin \alpha - \sin \beta)}, \\
		&q_{lsl} = \beta - \arctan \frac{\cos \beta - \cos \alpha}{d + \sin \alpha - \sin \beta}\{\text{mod}~2\pi\}, \\
		&\mathcal{L}_{lsl} = t_{lsl} + p_{lsl} + q_{lsl} = -\alpha + \beta + p_{lsl}.
	\end{aligned}
\end{equation}

The length of each segment and the total length of \textit{RSR} are derived by applying $R_q(S_p(R_t(0, 0, \alpha))) = (d, 0, \beta)$:
\begin{equation}
	\begin{aligned}
		&t_{rsr} = \alpha + \arctan \frac{\cos \alpha - \cos \beta}{d - \sin \alpha + \sin \beta}\{\text{mod}~ 2\pi\}, \\
		&p_{rsr} = \sqrt{2 + d^2 - 2\cos(\alpha - \beta) - 2d(\sin \alpha - \sin \beta)}, \\
		&q_{rsr} = -\beta(\text{mod}~2\pi) + \arctan \frac{\cos \alpha - \cos \beta}{d - \sin \alpha + \sin \beta}\{\text{mod}~ 2\pi\}, \\
		&\mathcal{L}_{rsr} = t_{rsr} + p_{rsr} + q_{rsr} = \alpha - \beta + p_{rsr}.
	\end{aligned}
\end{equation}

The length of each segment and the total length of \textit{LSR} are derived by applying $L_q(S_p(R_t(0, 0, \alpha))) = (d, 0, \beta)$:
\begin{equation}
	\begin{aligned}
		&t_{lsr} = (-\alpha + \arctan(\frac{- \cos \alpha - \cos \beta}{d + \sin \alpha + \sin \beta}) \\
		&\qquad \ \ + \arctan(\frac{2}{p_{lsr}}))\{\text{mod}~ 2\pi\}, \\
		&p_{lsr} = \sqrt{-2 + d^2 + 2\cos(\alpha - \beta) + 2d(\sin \alpha + \sin \beta)}, \\
		&q_{lsr} = -\beta(\text{mod}~2\pi) + \arctan (\frac{- \cos \alpha - \cos \beta}{d + \sin \alpha + \sin \beta}) \\
		&\qquad \ \ - \arctan (\frac{-2}{p_{lsr}})\{\text{mod}~ 2\pi\}, \\
		&\mathcal{L}_{lsr} = t_{lsr} + p_{lsr} + q_{lsr} = \alpha - \beta + 2t_{lsr} + p_{lsr}.
	\end{aligned}
\end{equation}

The length of each segment and the total length of \textit{RSL} are derived by applying $R_q(S_p(L_t(0, 0, \alpha))) = (d, 0, \beta)$:
\begin{equation}
	\begin{aligned}
		&t_{rsl} = \alpha - \arctan(\frac{\cos \alpha + \cos \beta}{d - \sin \alpha - \sin \beta}) \\
		&\qquad \ \ + \arctan(\frac{2}{p_{rsl}})\{\text{mod}~ 2\pi\}, \\
		&p_{rsl} = \sqrt{-2 + d^2 + 2\cos(\alpha - \beta) - 2d(\sin \alpha + \sin \beta)}, \\
		&q_{rsl} = \beta(\text{mod}~2\pi) - \arctan (\frac{\cos \alpha + \cos \beta}{d - \sin \alpha - \sin \beta}), \\
		&\qquad \ \ + \arctan (\frac{2}{p_{rsl}})\{\text{mod}~ 2\pi\}, \\
		&\mathcal{L}_{rsl} = t_{rsl} + p_{rsl} + q_{rsl} = - \alpha + \beta + 2t_{rsl} + p_{rsl}.
	\end{aligned}
\end{equation}

The length of each segment and the total length of \textit{RLR} are derived by applying $R_q(L_p(R_t(0, 0, \alpha))) = (d, 0, \beta)$:
\begin{equation}
	\begin{aligned}
		&t_{rlr} = \alpha - \arctan(\frac{\cos \alpha - \cos \beta}{d - \sin \alpha + \sin \beta}) + \frac{2}{p_{rlr}}\{\text{mod}~ 2\pi\}, \\
		&p_{rlr} = \arccos \frac{1}{8}(6 - d^2 + 2\cos(\alpha - \beta) + 2d(\sin \alpha - \sin \beta)), \\
		&q_{rlr} = \alpha - \beta - t_{rlr} + p_{rlr}\{\text{mod}~ 2\pi\}, \\
		&\mathcal{L}_{rlr} = t_{rlr} + p_{rlr} + q_{rlr} = \alpha - \beta + 2 p_{rlr}.
	\end{aligned}
\end{equation}

The length of each segment and the total length of \textit{LRL} are derived by applying $L_q(R_p(L_t(0, 0, \alpha))) = (d, 0, \beta)$:
\begin{equation}
	\begin{aligned}
		&t_{lrl} = (- \alpha + \arctan(\frac{-\cos \alpha + \cos \beta}{d + \sin \alpha - \sin \beta}) + \frac{2}{p_{lrl}})\{\text{mod}~ 2\pi\}, \\
		&p_{lrl} = \arccos \frac{1}{8}(6 - d^2 + 2\cos(\alpha - \beta) \\
		&\qquad \ \ + 2d(\sin \alpha - \sin \beta))\{\text{mod}~2\pi\}, \\
		&q_{lrl} = - \alpha + \beta(\text{mod}~2\pi) + p_{lrl}\{\text{mod}~ 2\pi\}, \\
		&\mathcal{L}_{lrl} = t_{lrl} + p_{lrl} + q_{lrl} = \alpha - \beta + 2 p_{lrl}.
	\end{aligned}
\end{equation}

\subsection{Equivalency Group}
\label{subsec:Equivalency Group}

The range of possible orientation angles can be partitioned into four quadrants. Each combination of $\alpha$ and $\beta$ is defined as $a_{ij}$, where index $i$ corresponds to the quadrant number of the initial, and index $j$ that of the final orientation. For instance, the case where $\alpha \in [0, \pi/2], \beta \in [\pi, 3\pi/2]$ corresponds to $a_{13}$. 

To categorize the $a_{ij}$ classes, a key property of the Dubins paths is utilized. For any path connecting two configurations, with orientation angles $(\alpha, \beta)$, there exist three other paths that are topologically equivalent. These equivalent paths correspond to the orientation angle pairs $(-\alpha, -\beta)$, $(\beta, \alpha)$, and $(-\beta, -\alpha)$ \cite{shkel2001classification}. Consequently, the $a_{ij}$ classes can be grouped into six independent clusters, termed equivalency groups: $\mathbb{E}_1 = \{a_{11}, a_{44}\}$, $\mathbb{E}_2 = \{a_{12}, a_{21}, a_{34}, a_{43}\}$, $\mathbb{E}_3 = \{a_{13}, a_{31}, a_{24}, a_{42}\}$, $\mathbb{E}_4 = \{a_{14}, a_{41}\}$, $\mathbb{E}_5 = \{a_{22}, a_{33}\}$, $\mathbb{E}_6 = \{a_{23}, a_{32}\}$. To illustrate, consider $a_{13}$ as an example. This class corresponds to the case where $\alpha \in [0, \pi/2]$, $\beta \in [\pi, 3\pi/2]$. Subsequently, $-\alpha \in [-\pi/2, 0]$ and $-\beta \in [-3\pi/2, -\pi]$. The pairs $(-\alpha, -\beta)$, $(\beta, \alpha)$, and $(-\beta, -\alpha)$ correspond to $a_{42}$, $a_{31}$, and $a_{24}$, respectively. 

Classes within the same equivalency group exhibit similar properties. In a Dubins path, the initial and final segments are circular arcs, denoted $C_i$ and $C_f$, respectively. The conjugate of a circular arc $C$ is defined as $\bar{C}$; for example, $\bar{C_{ir}} = C_{il}$, and vice versa. A key theorem related to equivalency groups, which substantially reduces computational cost, states: Given a Dubins path of form $C_i C_f [\alpha, \beta] (t, q)$, there exist three topologically equivalent paths: $\bar{C_i}\bar{C_f}[-\alpha, -\beta](t, q)$, $\bar{C_f}\bar{C_i}[\beta, \alpha](q, t)$, and $C_fC_i[-\beta, -\alpha](q, t)$ \cite{shkel2001classification}.

Once the solution for a single class is determined, the solution for the remaining classes in the group can be efficiently derived using orthogonal transformation. An illustrative example is presented in the subsequent section.

\section{Extraction Method From the Dubins Paths} \label{sec:Complete Calculation of Dubins Curve}

We now prepare to elaborate on how to identify the shortest path from the Dubins set for short distance between initial and final configurations. First, we introduce a key property of \textit{CCC} typed paths, which are feasible only in short path cases, to simplify the subsequent proofs. Subsequently, we derive the corresponding \textit{switching functions} for each class, which uniquely determine the shortest path. Finally, the shortest paths for other classes within the same equivalency group can be efficiently obtained through an orthogonal transformation.

\subsection{The Necessary Condition for \textit{CCC} Typed Path Being the Shortest}

In the short path case, a path of type \textit{CCC} may yield the shortest solution, making it necessary to analyze its properties in detail.

Notably, the initial and final segments of a \textit{CCC} path are arcs with the same orientation. The middle segment is arc with the opposite orientation, tangent to both the two arcs. As illustrated by the solid line in Fig. \ref{fig:PossibleCCCPath}, two candidate arcs exist for the middle segment. Determining which of the two arcs results in the shorter path is therefore essential. 

The centers of the four relevant circles form a rhombus. The relative positions of the initial point $P_i$ and final point $P_f$ with respect to this rhombus provide insight into the choice of the middle arc. If both $P_i$ and $P_f$ lie inside or outside the rhombus, they are considered to be on the same side; otherwise, they are on the opposite sides, as illustrated in Fig. \ref{fig:relative position}.

 \begin{figure}[htbp]
     \centering
     \subfigure[Case of a possible \textit{CCC} path]{
         \includegraphics[width=3cm]{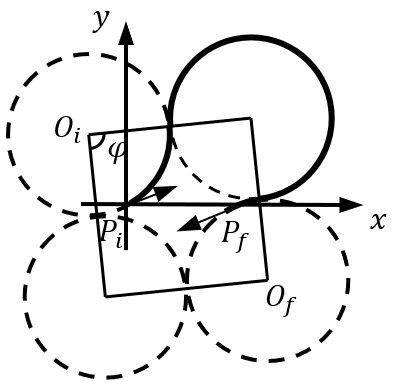}
         \label{fig:PossibleCCCPath1}}
     \subfigure[Case of another possible \textit{CCC} path]{
         \includegraphics[width=3cm]{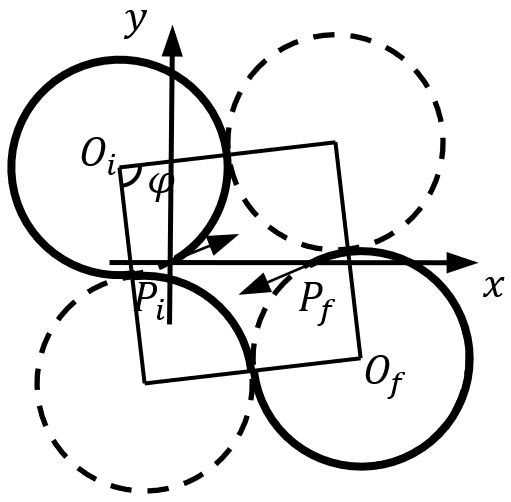}
         \label{fig:PossibleCCCPath2}}
     \caption{Cases of two possible \textit{CCC} paths for a configuration}
     \label{fig:PossibleCCCPath}
 \end{figure}

 \begin{figure}[htbp]
     \centering
     \subfigure[Initial and final points are on the same side]{
         \includegraphics[width=2.9cm]{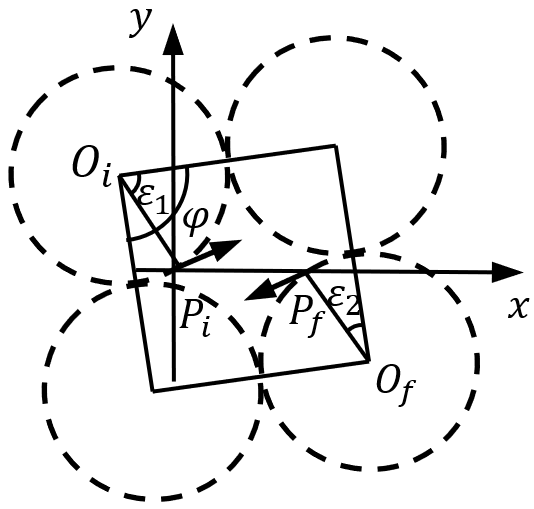}
         \label{fig:InIn}}
     \subfigure[Initial and final points are on the opposite side]{
         \includegraphics[width=3.3cm]{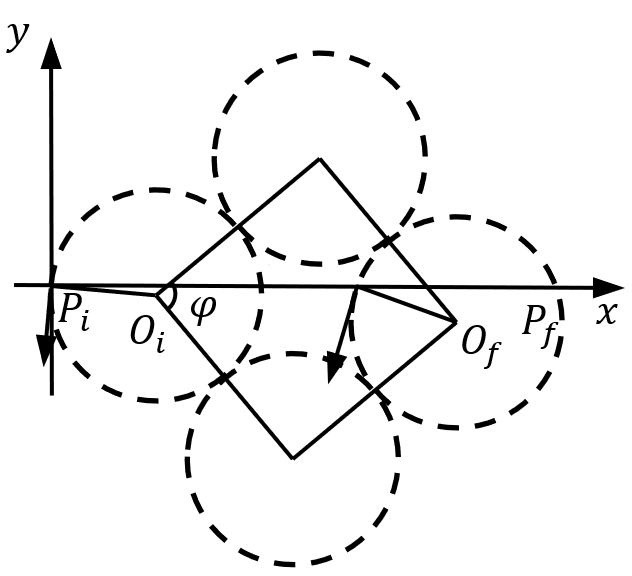}
         \label{fig:OutIn}}
     \caption{Relative position of the rhombus with respect to the initial and final configurations}
	 \label{fig:relative position}
 \end{figure}

\begin{proposition} \label{prop:Necessary Condition of CCC Being Optimal Path}
  The necessary conditions for \textit{CCC} typed path being the shortest are: 1) the initial and final points are on the opposite sides of the rhombus and $\varphi<\pi/2$, or 2) they are on the same side of the rhombus.
\end{proposition}

\begin{proof}
 Let $c_1$ and $c_2$ denote the lengths of the curves in Fig. \ref{fig:PossibleCCCPath}, respectively.
 
 Consider the case where the initial and final points are on the same side first. If they are both inside, as shown in Fig. \ref{fig:InIn}, the lengths of two paths are
 \begin{equation}
	\begin{aligned}
     c_1 &= \varepsilon_1 + (\alpha + \pi) + \varepsilon_2, \\
     c_2 &= [2\pi + (\alpha - \varepsilon_1)] + (\pi - \alpha) + (2\pi - \alpha + \varepsilon_2), \\
   \end{aligned}
 \end{equation}
 and the difference is
 \begin{equation}
     c_2 - c_1 = 4(\pi - \alpha) > 0.
 \end{equation}
 If they are both outside, the difference of the lengths is
 \begin{equation}
	\begin{aligned}
     c_2 - c_1 = -4\varphi < 0.
   \end{aligned}
 \end{equation}

 Therefore, if the $P_i$ and $P_f$ are both inside, the path with the middle arc as an minor arc cannot be the shortest curve, as the alternative is shorter. Conversely, if both points are outside, the path with the middle arc as a major arc cannot be the shortest curve.
 
 When the points are on the opposite sides, the difference of the path length is
 \begin{align}
   c_2 - c_1 = 4(\frac{\pi}{2} - \varphi).
 \end{align}
 
 Therefore, if $\varphi > \pi/2$, the shortest path cannot be the one with the middle arc as a major arc, and if $\varphi < \pi/2$, the shortest path cannot be the one with the middle arc as a minor arc.
 
 The \textit{CCC} typed path with the middle arc as a major arc cannot be an optimal solution \cite{dubins1957curves}. Therefore, \textit{CCC} cannot be the shortest path when the initial and final points are on the opposite side with $\varphi < \pi/2$ or both points are outside.
\end{proof}

We now proceed to define the equivalency groups for each class and present approaches for extracting the shortest path within each equivalency group. The following additional notation is used: forms like $AB$ and $\wideparen{AB}$ denote straight line segments and circular arc segments, respectively, with $A$ and $B$ being the segments' endpoints.

\subsection{Equivalency Group $\mathbb{E}_1$}
As mentioned previously, classes $a_{11}$ and $a_{44}$ belong to the same equivalency group. We demonstrate how to extract the shortest path for class $a_{11}$. By applying an orthogonal transformation to class $a_{11}$, the shortest path for class $a_{44}$ can be obtained.

\begin{theorem} \label{theo:a_11}
  For the short path case, the shortest path corresponding to the class $a_{11}$ may be \textit{LSL, RSR, RSL, LSR, RLR, LRL}, as shown in Table \ref{tab:a_11OptimalPath}.
  \begin{table}[htbp]
    \caption{Shortest Path Corresponding to Class $a_{11}$}
    \label{tab:a_11OptimalPath}
    \renewcommand{\arraystretch}{1.3}
    \begin{center}
      \begin{tabular}{c|c|c|c} 
        \Xhline{1.2pt}
		\multicolumn{3}{c|}{Condition} & Shortest Path \\ 
		\hline
		\multicolumn{3}{c|}{$C_{ir} \cap C_{fl} = \emptyset$} & RSL \\
		\hline
		\multirow{6}*{$C_{ir} \cap C_{fl} \ne \emptyset$} & \multirow{3}*{$\alpha < \beta$} & $t_{rsr} < \pi, S_{11}^1 > 0$ & RSR \\
		~							 &~								& $t_{rsr} > \pi, S_{11}^2 > 0$ & LSR \\
		~							 &~								&     Otherwise					  & RLR \\
		\cline{2-4}
		~							 &\multirow{3}*{$\alpha > \beta$} & $q_{lsl} < \pi, S_{11}^3 > 0$ & LSL \\
		~							 &~								& $q_{lsl} > \pi, S_{11}^4 > 0$ & LSR \\
		~							 &~								&     Otherwise					  & LRL \\ 
		\Xhline{1.2pt}
	  \end{tabular}
    \end{center}
  \end{table}

  The switching functions in Table \ref{tab:a_11OptimalPath} are
  \begin{equation}
    \begin{aligned}
      S_{11}^1 &= 2(p_{rlr} - \pi) - p_{rsr}, \\
	  S_{11}^2 &= 2(t_{rlr} + q_{rlr}) - (p_{lsr} + 2q_{lsr}) + 2\pi, \\
	  S_{11}^3 &= 2(p_{lrl} - \pi) - p_{lsl}, \\
	  S_{11}^4 &= 2(t_{lrl} + q_{lrl}) - (p_{lsr} + 2t_{lsr}) + 2\pi.
    \end{aligned}
  \end{equation}
\end{theorem}

\begin{proof}
 Since $\mathcal{L}_{rsl} < \min \{\mathcal{L}_{lsl}, \mathcal{L}_{rsr}, \mathcal{L}_{lsr}\}$ \cite{shkel2001classification}, the shortest path is \textit{RSL}, if \textit{RSL} typed path is feasible, i.e.,
 \begin{align}
   d_{rl} < d < d_{0}.
 \end{align}
 
 Due to symmetry, if $\alpha$ equals to $\beta$, then $\mathcal{L}_{lsl} = \mathcal{L}_{rsr}$ and $\mathcal{L}_{lrl} = \mathcal{L}_{rlr}$ hold. Furthermore, based on the aforementioned length of each path, we obtain $\frac{\partial \mathcal{L}_{lsl}}{\partial \alpha} < 0$, $\frac{\partial \mathcal{L}_{lsl}}{\partial \beta} > 0$ and $\frac{\partial \mathcal{L}_{rsr}}{\partial \alpha} > 0$, $\frac{\partial \mathcal{L}_{rsr}}{\partial \beta} < 0$. Therefore, if $\alpha < \beta$, then $\mathcal{L}_{rsr} < \mathcal{L}_{lsl}$ and $\mathcal{L}_{rlr} < \mathcal{L}_{lrl}$; whereas if $\alpha > \beta$, then $\mathcal{L}_{rsr} > \mathcal{L}_{lsl}$ and $\mathcal{L}_{rlr} > \mathcal{L}_{lrl}$. 
 
 Consider the relationship between \textit{RLR}, \textit{RSR}, and \textit{LSR} if $\alpha$ is less than $\beta$ first.
 
 We first derive the switching functions which determine the path \textit{RSR} and \textit{LSR}. Define the critical initial orientation as the orientation $\alpha$ that coincides with the tangent to the circle $O_{fr}$, as shown in Fig. \ref{fig:CriticalInitialOrientation}, and denote it as $\alpha = \overline{\alpha}$. The critical orientation $\overline{\alpha}$ is uniquely defined by the set $(\beta, d)$. The degenerated path $P_i\wideparen{TP_f}$ corresponds to \textit{SR}, which implies that if $\alpha > \overline{\alpha}$, then the path \textit{LSR} can be excluded from consideration ($\mathcal{L}_{lsr} > \mathcal{L}_{rsr}$); if $\alpha < \overline{\alpha}$, the path \textit{RSR} can be excluded. Specifically, if $\alpha > \overline{\alpha}$ then $t_{lsr} > \pi$ and $t_{rsr} < \pi$; if $\alpha < \overline{\alpha}$, then $t_{lsr} < \pi$ and $t_{rsr} > \pi$. Thus, $t_{rsr}$ or $t_{lsr}$ can serve as the switching function. If $t_{rsr} > \pi$, then $\mathcal{L}_{rsr} > \mathcal{L}_{lsr}$; if $t_{rsr} < \pi$, then $\mathcal{L}_{rsr} < \mathcal{L}_{lsr}$.

  \begin{figure}[htbp]
     \centering
     \subfigure[Critical initial orientation]{
         \includegraphics[width=3cm]{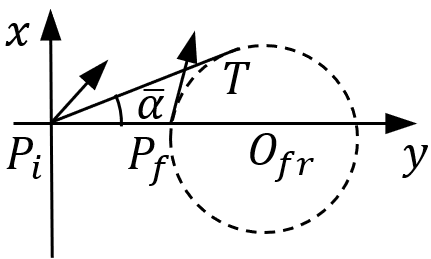}
         \label{fig:CriticalInitialOrientation}}
     \subfigure[Critical final orientation]{
         \includegraphics[width=3.1cm]{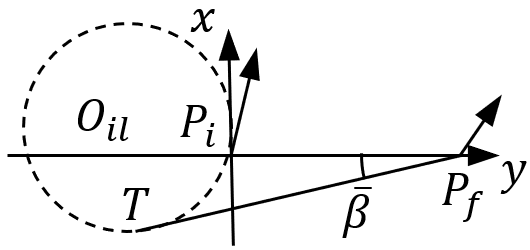}
         \label{fig:CriticalFinalOrientation}}
     \caption{Critical orientation}
 \end{figure}

 The derivation of switching functions which determine the path \textit{RLR} and \textit{RSR} is more complex. As shwon in Fig. \ref{fig:a11_RLR_RSR}, the path $\wideparen{P_iA}\wideparen{BP_f}$ is the \textit{RSR} path and $\wideparen{P_iDEP_f}$ is the \textit{RLR} path. The circle $O_{tl}$ is the common tangent circle of $O_{ir}$ and $O_{fr}$, with $AO_{ir} // CO_{tl} // BO_{fr}$ and $FP_f // GH$. So $\mathcal{L}_{rlr}$, $\mathcal{L}_{rsr}$ and the switching function $S_{11}^1$ respectively are
 \begin{equation}
		\begin{aligned}
			\mathcal{L}_{rlr} &= (\xi + \gamma) + [(\gamma - \zeta) + \pi + \sigma] + \sigma, \\
			\mathcal{L}_{rsr} &= \xi + p_{rsr} + \zeta + \pi, \\
			S_{11}^1 &= \mathcal{L}_{rlr} - \mathcal{L}_{rsr} = 2(p_{rlr} - \pi) - p_{rsr}.
		\end{aligned}
	\end{equation}

 If $S_{11}^1 > 0$ then $\mathcal{L}_{rlr} > \mathcal{L}_{rsr}$; conversely, $\mathcal{L}_{rlr} < \mathcal{L}_{rsr}$.

  \begin{figure}[htbp]
     \centering
     \subfigure[An illustration the switching function $S_{11}^1$ for class $a_{11}$]{
         \includegraphics[width=3cm]{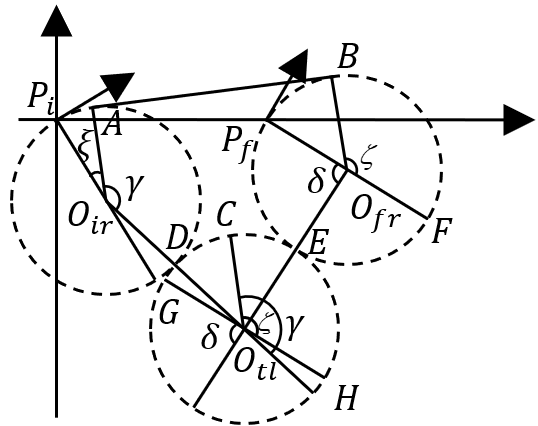}
         \label{fig:a11_RLR_RSR}}
     \subfigure[An illustration of switching function $S_{11}^2$ for class $a_{11}$]{
         \includegraphics[width=3cm]{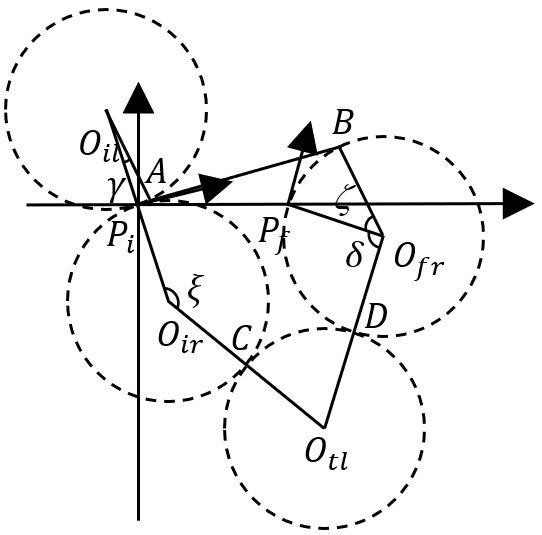}
         \label{fig:a11_RLR_LSR}}
     \caption{Illustrations of switching functions for $a_{11}$}
 \end{figure}

 The derivation of switching functions which determine \textit{RLR} and \textit{LSR} are similar with that of $S_{11}^1$. As illustrated in Fig. \ref{fig:a11_RLR_LSR}, the path lengths $\mathcal{L}_{rlr}$, $\mathcal{L}_{lsr}$, along with the corresponding switching function $S_{11}^2$ are given by:
 \begin{equation}
	\begin{aligned}
	  \mathcal{L}_{rlr} &= \xi + \xi + \gamma + \zeta + \delta + \delta, \\
	  \mathcal{L}_{lsr} &= \gamma + p_{lsr} + 2\pi - \zeta, \\
	  S_{11}^2 &= \mathcal{L}_{rlr} - \mathcal{L}_{lsr} \\
                &= 2(t_{rlr} + q_{rlr}) - (p_{lsr} + 2q_{lsr}) + 2\pi.
	\end{aligned}
 \end{equation}
 
 If $S_{11}^2 > 0$, then $\mathcal{L}_{rlr} > \mathcal{L}_{lsr}$; if $S_{11}^2 < 0$ $\mathcal{L}_{rlr} < \mathcal{L}_{lsr}$.

 Similarly, if $\alpha$ is greater than $\beta$, the switching functions to determine \textit{LRL, LSL} and \textit{LSR} are: $S_{11}^3 = 2(p_{lrl} - \pi) - p_{lsl}$, $S_{11}^4 = 2(t_{lrl} + q_{lrl}) - (p_{lsr} + 2t_{lsr}) + 2\pi$.
 
 It is noteworthy that the switching function for distinguishing between \textit{LSL} and \textit{LSR} is $q_{lsl}$, as they share a common degenerated path \textit{LS}. Define the critical final orientation as one where the orientation $\beta$ aligns with the tangent to the circle $O_{il}$ (see Fig. \ref{fig:CriticalFinalOrientation}). This critical orientation is denoted by $\beta = \overline{\beta}$ and is uniquely defined by the set $(\alpha, d)$. Consequently, \textit{LSR} can be excluded from consideration if $\beta > \overline{\beta}$ ($\mathcal{L}_{lsr} > \mathcal{L}_{lsl}$); conversely, if $\beta < \overline{\beta}$, then the path \textit{LSL} can be excluded. Thus, $t_{rsr}$ serves as the switching function when $\alpha<\beta$, whereas $q_{lsl}$ is applied when $\alpha>\beta$.
\end{proof}

\subsection{Equivalency Group $\mathbb{E}_2$}
Classes $a_{12}$, $a_{21}$, $a_{34}$, and $a_{43}$ belong to the same equivalency group. We demonstrate how to extract the shortest path for class $a_{12}$ as an example.

\begin{theorem} \label{theo:a_12}
  For the short path case, the shortest path corresponding to the class $a_{12}$ may be \textit{RSR, LSR, RLR, LRL}, as shown in Table \ref{tab:a_12OptimalPath}.
  
  \begin{table}[htbp]
    \caption{Shortest Path Corresponding to Class $a_{12}$}
    \label{tab:a_12OptimalPath}
    \renewcommand{\arraystretch}{1.3}
    \begin{center}
      \begin{tabular}{c|c|c|c}
		\Xhline{1.2pt}
		\multicolumn{3}{c|}{Condition}    &  Shortest Path\\
		\hline
		\multirow{4}*{$C_{il} \cap C_{fr} = \emptyset$} 	& \multirow{2}*{$t_{rsr} < \pi$} 	 & $S_{12}^1 < 0$ & RLR \\
		~									& ~ 								 & $S_{12}^1 > 0$ & RSR \\
		\cline{2-4}
		~									& \multirow{2}*{$t_{rsr} > \pi$} 	 & $S_{12}^2 < 0$ & RLR \\
		~									& ~ 								 & $S_{12}^2 > 0$ & LSR \\
		\hline
		\multicolumn{3}{c|}{$C_{il} \cap C_{fr} \ne \emptyset$}													 & LRL \\
		\Xhline{1.2pt}
	  \end{tabular}
    \end{center}
  \end{table}
  the switching functions in Table \ref{tab:a_12OptimalPath} are
  \begin{equation}
	\begin{aligned}
	  S_{12}^1 &= 2(p_{rlr} - \pi) - p_{rsr}, \\
	  S_{12}^2 &= 2(t_{rlr} + q_{rlr}) - (p_{lsr} + 2q_{lsr}) + 2\pi.
	\end{aligned}
  \end{equation}
  
\end{theorem}

\begin{proof}
  Since $C_{ir} \cap C_{fl} \ne \emptyset$ occurs in the short path case, the path \textit{RSL} is infeasible. As shown in Fig. \ref{fig:a_12_LSLImpossible}, the length of the first segment of \textit{LSL} exceeds $\pi$, leading to $\mathcal{L}_{lsl} > |P_i\wideparen{TP_f}|$. This implies that if $\alpha<\overline{\alpha}$, then $\mathcal{L}_{lsr} < \mathcal{L}_{lsl}$; conversely if $\alpha>\overline{\alpha}$ then $\mathcal{L}_{lsr} > \mathcal{L}_{lsl}$. Therefore, the \textit{LSL} path can also be excluded.
 
 \begin{figure}[htbp]
		\centering
		\includegraphics[width=3cm]{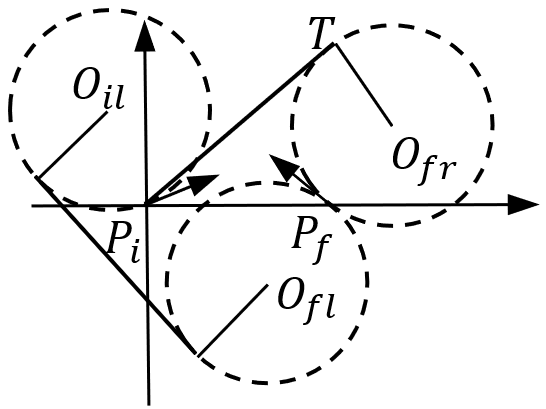}
		\caption{Case where \textit{LSL} cannot be the shortest path for class $a_{12}$}
		\label{fig:a_12_LSLImpossible}
	\end{figure}

 The candidate solutions are \textit{RLR}, \textit{RSR}, and \textit{LSR} if $C_{il}$ and $C_{fr}$ do not intersect. Similar to class $a_{11}$, $t_{rsr}$ (or $t_{lsr}$) serves as the switching function which determines \textit{RSR} and \textit{LSR} with $\overline{\alpha}$ as the critical initial orientation. Subsequently, their respective path lengths are compared with that of \textit{RLR}.
 
 If all the four circles intersect, the path \textit{LSR} becomes infeasible. The lengths of initial and final segments of \textit{RSR} both exceed $\pi$, making \textit{RSR} significantly longer than other candidates. Consequently, the \textit{CSC} typed path cannot be the shortest path. Furthermore, since $\mathcal{L}_{rlr}$ is greater than $\mathcal{L}_{lrl}$, the shortest path is \textit{LRL}.
\end{proof}

\subsection{Equivalency Group $\mathbb{E}_3$}
Classes $a_{13}$, $a_{31}$, $a_{24}$ and $a_{42}$ belong to the same equivalency group. We demonstrate how to extract the shortest path for class $a_{13}$ as an example.

\begin{theorem} \label{theo:a_13}
  For the short path case, the shortest path corresponding to the class $a_{13}$ may be \textit{RSR, RSL, LSR, RLR, LRL}, as shown in Table \ref{tab:a_13OptimalPath}.
  
  \begin{table}[htbp]
    \caption{Shortest Path Corresponding to Class $a_{13}$}
    \label{tab:a_13OptimalPath}
    \renewcommand{\arraystretch}{1.3}
    \begin{center}
      \begin{tabular}{c|c|c|c}
      	\Xhline{1.2pt}
	  	\multicolumn{3}{c|}{Condition}		    & Shortest Path \\
	  	\hline
	  	\multirow{4}*{$C_{il} \cap C_{fr} \ne \emptyset$} 	& \multirow{2}*{$C_{ir} \cap C_{fl} \ne \emptyset$} 			& $S_{13}^1 < 0$ & RLR \\
	  	~								& ~ 									& $S_{13}^1 > 0$ & LRL \\
	  	\cline{2-4}
	  	~								& \multirow{2}*{$C_{ir} \cap C_{fl} = \emptyset$} 			& $S_{13}^2 < 0$ & LRL \\
	  	~								& ~ 							 		& $S_{13}^2 > 0$ & RSL \\
	  	\hline
	  	\multirow{2}*{$C_{il} \cap C_{fr} = \emptyset$}		& \multicolumn{2}{c|}{$t_{rsr} < \pi$}					& RSR \\
	  	~								& \multicolumn{2}{c|}{$t_{rsr} > \pi$}					& LSR \\
		\Xhline{1.2pt}
	  \end{tabular}
    \end{center}
  \end{table}
  the switching functions in Table \ref{tab:a_13OptimalPath} are
  \begin{equation}
	\begin{aligned}
	  S_{13}^1 &= \alpha - \beta + p_{rlr} - p_{lrl}, \\
	  S_{13}^2 &= 2(q_{lrl} + t_{lrl}) + 2q_{rsl} - p_{rsl} - 2\pi.
	\end{aligned}
  \end{equation}  
\end{theorem}

\begin{proof}
 For class $a_{13}$, the distance between $C_{il}$ and $C_{fl}$ is greater than that between $C_{ir}$ and $C_{fr}$. In this case, when $\alpha = 0$ and $\beta = \pi$, the path lengths satisfy $\mathcal{L}_{lsl} = \mathcal{L}_{rsr}$. Thus, $\mathcal{L}_{lsl}$ becomes greater than $\mathcal{L}_{rsr}$ if $\alpha > 0$ and $\beta > \pi$. Hence, \textit{LSL} can be excluded from the list of candidates for the shortest path.
 
 Similar to the long path case, if \textit{LSR} and \textit{RSR} are feasible (i.e., $C_{il} \cap C_{fr} = \emptyset$), the shortest path is one of these two. Either $t_{rsr}$ or $t_{lsr}$ can serve as the switching function.
 
 For the case where $C_{il}$ intersects with $C_{fr}$, if $C_{ir}$ and $C_{fl}$ are tangent, then \textit{RSL} and \textit{RLR} coincide, as the lengths of the second segment \textit{S} of \textit{RSL} and the final segment \textit{R} of \textit{RLR} both reduce to zero. Thus, the resulting path degenerated into their common curve \textit{RL}. Accordingly, if $C_{ir} \cap C_{fl} = \emptyset$, then $\mathcal{L}_{rsl}<\mathcal{L}_{rlr}$. Similarly, \textit{RSR} can be excluded from consideration. The remaining candidate paths are \textit{RSL} and \textit{LRL}. As shown in Fig. \ref{fig:a13_LRL_RSL}, their lengths and the switching function $S_{13}^2$ respectively are
 \begin{equation}
	\begin{aligned}
	  \mathcal{L}_{lrl} &= 2\gamma + \xi + \delta + 2\zeta, \\
	  \mathcal{L}_{rsl} &= \xi + p_{rsl} + 2\pi - \zeta, \\
	  S_{13}^2 &= \mathcal{L}_{lrl} - \mathcal{L}_{rsl} \\
			   &= 2(t_{lrl} + q_{lrl}) + 2q_{rsl} - p_{rsl} - 2\pi.
	\end{aligned}
 \end{equation}
 
 \begin{figure}
   \centering
   \includegraphics[width=3cm]{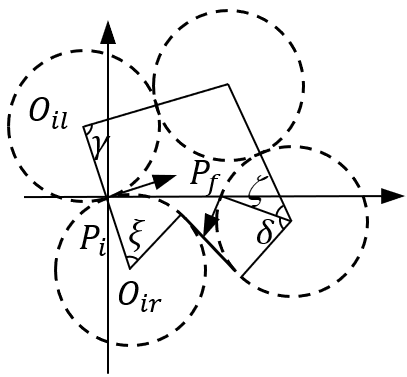}
   \caption{An illustration of switching function $S_{13}^2$ for class $a_{13}$}
   \label{fig:a13_LRL_RSL}
 \end{figure}
 
 Since \textit{RSL} is infeasible if $C_{ir} \cap C_{fl} \ne \emptyset$, we now only need to compare $\mathcal{L}_{lrl}$ and $\mathcal{L}_{rlr}$. The switching function is
 \begin{equation}
   \begin{aligned}
		S_{13}^1 &= \frac{\mathcal{L}_{rlr} - \mathcal{L}_{lrl}}{2} \\ 
				 &= \alpha - \beta + p_{rlr} - p_{lrl}.
	\end{aligned}
	\end{equation}
\end{proof}

\subsection{Equivalency Group $\mathbb{E}_4$}
Classes $a_{14}$, $a_{41}$ belong to the same equivalency group. We demonstrate how to extract the shortest path for class $a_{14}$ as an example.

\begin{theorem} \label{theo:a_14}
  For the short path case, the shortest path corresponding to the class $a_{14}$ may be \textit{RSR, RSL, LSR, RLR, LRL}, as shown in Table \ref{tab:a_14OptimalPath}.
  
  \begin{table}[htbp]
    \caption{Shortest Path Corresponding to Class $a_{14}$}
    \label{tab:a_14OptimalPath}
    \renewcommand{\arraystretch}{1.3}
    \begin{center}
      \begin{tabular}{c|c|c} 
    	\Xhline{1.2pt}
		\multicolumn{2}{c|}{Condition}			& Shortest Path \\
		\hline
		\multirow{3}*{\makecell[c]{$C_{il} \cap C_{fr} = \emptyset$, \\ $C_{ir} \cap C_{fl} = \emptyset$}}	& $t_{rsr} > \pi$ & LSR \\
		~										& $q_{rsr} > \pi$ & RSL \\
		~										& Otherwise		 & RSR \\
		\hline
		\multirow{2}*{\makecell[c]{$C_{il} \cap C_{fr} \ne \emptyset$, \\ $C_{ir} \cap C_{fl} = \emptyset$}}		& $S_{14}^1 < 0$  & LRL \\
		~										& $S_{14}^1 > 0$	 & LSR \\
		\hline
		\multirow{2}*{\makecell[c]{$C_{il} \cap C_{fr} = \emptyset$, \\ $C_{ir} \cap C_{fl} \ne \emptyset$}}		& $S_{14}^2 < 0$  & LRL \\
		~										& $S_{14}^2 > 0$	 & RSL \\
		\hline
		\multirow{2}*{\makecell[c]{$C_{il} \cap C_{fr} \ne \emptyset$, \\ $C_{ir} \cap C_{fl} \ne \emptyset$}}			& $S_{14}^3 < 0$  & RLR \\
		~										& $S_{14}^3 > 0$  & LRL \\
		\Xhline{1.2pt}
	  \end{tabular}
    \end{center}
  \end{table}
  the switching functions in Table \ref{tab:a_14OptimalPath} are
  \begin{equation}
	\begin{aligned}
	  S_{14}^1 &= 2(t_{lrl} + q_{lrl}) + 2t_{lsr} - p_{lsr} - 2\pi, \\
	  S_{14}^2 &= 2(t_{lrl} + q_{lrl}) + 2q_{rsl} - p_{rsl} - 2\pi, \\
	  S_{14}^3 &= \alpha - \beta + p_{rlr} - p_{lrl}.
	\end{aligned}
  \end{equation}
\end{theorem}

\begin{proof}
 For class $a_{14}$, $C_{ir}$ and $C_{fr}$ intersect first. The analysis of shortest path for this class follows a similar approach to that in the long path case. As $P_i$ and $P_f$ move closer, two possible scenarios arise: either $C_{il}$ and $C_{fr}$ intersect first, or $C_{ir}$ and $C_{fl}$ intersect first, as illustrated in Fig. \ref{fig:CIntersectC}. 

  \begin{figure}[htbp]
     \centering
     \subfigure[Case where $C_{il}$ and $C_{fr}$ intersect first]{
         \includegraphics[width=3cm]{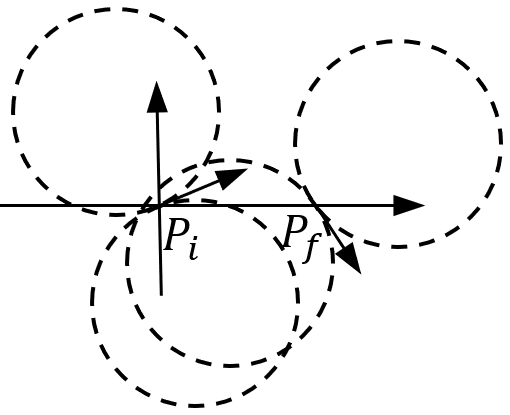}
         \label{fig:CIntersectC1}}
     \subfigure[Case where $C_{ir}$ and $C_{fl}$ intersect first]{
         \includegraphics[width=3cm]{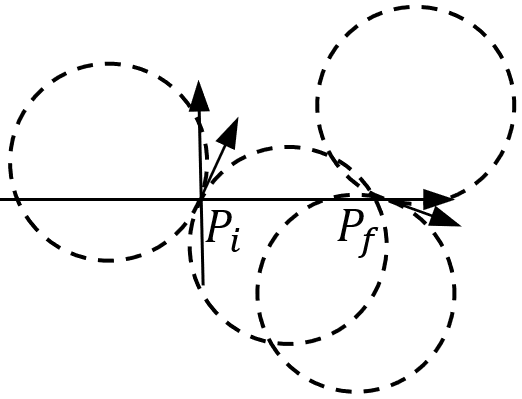}
         \label{fig:CIntersectC2}}
     \caption{Two cases if the initial and final positions get closer}
     \label{fig:CIntersectC}
 \end{figure}
 
 Consider the case where $C_{il}$ and $C_{fr}$ intersect first. It is evident that $P_i$ and $P_f$ lie outside the rhombus formed by $C_{ir}$ and $C_{fr}$. According to proposition \ref{prop:Necessary Condition of CCC Being Optimal Path}, \textit{RLR} path cannot be the shortest. The length of the initial and final segments of \textit{LSL} both exceed a length of $\pi$, leading to $\mathcal{L}_{lsl} > \mathcal{L}_{lsr}$. Similarly, $\mathcal{L}_{rsr}>\mathcal{L}_{lsr}$ under this condition, allowing both \textit{LSL} and \textit{RSR} to be excluded from consideration. The candidates for the shortest path are \textit{LRL} and \textit{LSR}. If $C_{ir}$ and $C_{fl}$ intersect first, the candidates are \textit{LRL} and \textit{RSL}.
 
 Furthermore, if $C_{il}$ intersect with $C_{fr}$ and $C_{ir}$ intersect with $C_{fl}$, then both \textit{LSR} and \textit{RSL} are infeasible. Additionally, the lengths of \textit{RSR} and \textit{LSL} exceed that of type \textit{CCC}. Thus, the shortest path is \textit{CCC} typed. The derivation of the switching functions follows the same procedure as described in the previous section. 
 
\end{proof}

\subsection{Equivalency Group $\mathbb{E}_5$}
Classes $a_{22}$ and $a_{33}$ belong to equivalency group $\mathbb{E}_5$. We demonstrate how to extract the shortest path for class $a_{22}$ as an example.

\begin{theorem} \label{theo:a_22}
  For the short path case, the shortest path corresponding to the class $a_{22}$ may be \textit{RSR, LSL, LSR, RLR, LRL}, as shown in Table \ref{tab:a_22OptimalPath}.
  
  \begin{table}[htbp]
    \caption{Shortest Path Corresponding to Class $a_{22}$}
    \label{tab:a_22OptimalPath}
    \renewcommand{\arraystretch}{1.3}
    \begin{center}
      \begin{tabular}{c|c|c} 
        \Xhline{1.2pt}
				\multicolumn{2}{c|}{Condition}									 	& Shortest Path \\
				\hline
  	  	\multirow{3}*{$\alpha < \beta$} 	& $q_{lsl} > \pi$					& LSR \\
  	  	~							    	& $q_{lsl} < \pi, S_{22}^1 > 0$ 	& RSR \\
  	  	~								& $q_{lsl} < \pi, S_{22}^1 < 0$	& RLR \\
   		 	\hline
   		 	\multirow{3}*{$\alpha > \beta$} 	& $t_{rsr} > \pi$					& LSR \\
    		~								& $t_{rsr} < \pi, S_{22}^2 > 0$	& LSL \\
    		~								& $t_{rsr} < \pi, S_{22}^2 < 0$	& LRL \\
				\Xhline{1.2pt}
	  	\end{tabular}
    \end{center}
  \end{table}
  the switching functions in Table \ref{tab:a_22OptimalPath} are
  \begin{equation}
	\begin{aligned}
	  S_{22}^1 &= 2(p_{rlr} - \pi) - p_{rsr}, \\
	  S_{22}^2 &= 2(p_{lrl} - \pi) - p_{lsl}.
	\end{aligned}
  \end{equation}
\end{theorem}

\begin{proof}
 Similar to class $a_{11}$, the equalities $\mathcal{L}_{rsr} = \mathcal{L}_{lsl}$ and $\mathcal{L}_{rlr} = \mathcal{L}_{lrl}$ hold, when $\alpha=\beta$. Hence, the comparison between $\mathcal{L}_{rsr}$ and $\mathcal{L}_{rlr}$ as well as that between $\mathcal{L}_{lsl}$ and $\mathcal{L}_{lrl}$ can be directly referenced from class $a_{11}$.
 
 We can also refer the derivation of the switching functions to determine the path \textit{LSR} and \textit{RSR}, as well as those for \textit{LSR} and \textit{LSL}, from class $a_{11}$. Let $\overline{\alpha}$ and $\overline{\beta}$ denote the critical orientation. The key distinction is that when $\alpha < \beta$, ensuring $\mathcal{L}_{lsr} < \mathcal{L}_{rsr}$ (i.e., $t_{rsr} > \pi$) requires $P_f$ to be positioned close to $P_i$, causing the angle $\varphi$ in Fig. \ref{fig:PossibleCCCPath} to exceed $\frac{\pi}{2}$. This indicates that $P_i$ and $P_f$ lie on the opposite side of the rhombus. According to proposition \ref{prop:Necessary Condition of CCC Being Optimal Path}, \textit{RLR} cannot be the shortest path. Similarly, if $\alpha > \beta$ and $q_{lsl} > \pi$, \textit{LRL} cannot be the shortest path.
 
\end{proof}

\subsection{Equivalency Group $\mathbb{E}_6$}
Classes $a_{23}$ and $a_{32}$ belong to equivalency group $\mathbb{E}_6$. We demonstrate how to extract the shortest path for class $a_{23}$ as an example.

\begin{theorem} \label{theo:a_23}
  For the short path case, the shortest path corresponding to the class $a_{23}$ may be \textit{RSR, RSL, LSR, LRL}, as shown in Table \ref{tab:a_23OptimalPath}.
  
  \begin{table}[htbp]
    \caption{Shortest Path Corresponding to Class $a_{23}$}
    \label{tab:a_23OptimalPath}
    \renewcommand{\arraystretch}{1.3}
    \begin{center}
      \begin{tabular}{c|c|c}
      	\Xhline{1.2pt}
				\multicolumn{2}{c|}{Condition}			& Shortest Path \\
				\hline
				\multicolumn{2}{c|}{\makecell[c]{$C_{il} \cap C_{fr} = \emptyset$, \\ $C_{ir} \cap C_{fl} = \emptyset$}}						& RSR \\
				\hline
				\multirow{2}*{\makecell[c]{$C_{il} \cap C_{fr} \ne \emptyset$, \\ $C_{ir} \cap C_{fl} = \emptyset$}}			& $q_{rsr} < 0$  & RSR \\
				~																			& $q_{rsr} > 0$		& RSL \\
				\hline
				\multirow{2}*{\makecell[c]{$C_{il} \cap C_{fr} = \emptyset$, \\ $C_{ir} \cap C_{fl} \ne \emptyset$}}			& $t_{rsr} < 0$  & RSR \\
				~																			& $t_{rsr} > 0$		& LSR \\
				\hline
				\multicolumn{2}{c|}{\makecell[c]{$C_{il} \cap C_{fr} \ne \emptyset$, \\ $C_{ir} \cap C_{fl} \ne \emptyset$}}									& LRL \\
				\Xhline{1.2pt}
			\end{tabular}
    \end{center}
	\end{table}  	
\end{theorem}

\begin{proof}
 Candidates \textit{LSL} and \textit{RLR} are excluded from consideration for class $a_{23}$. Since $\mathcal{L}_{rsr} < \mathcal{L}_{lsl}$ holds \cite{shkel2001classification}, \textit{LSL} also cannot be the shortest path. Additionally, $P_i$ and $P_f$ are both outside of the rhombus for candidate \textit{RLR}. According to proposition \ref{prop:Necessary Condition of CCC Being Optimal Path}, \textit{RLR} cannot be the shortest path.

 For class $a_{23}$, $C_{ir}$ and $C_{fr}$ are the first to intersect. If only these two circles intersect, the conclusion drawn for the long path case remain valid, making \textit{RSR} the shortest path in this scenario.

 When the distance between the initial and final positions is relatively large, the conditions $\alpha > \overline{\alpha}$ and $\beta < \overline{\beta}$ hold. However, this may not necessarily be the case if the distance is short. As $P_i$ and $P_f$ move closer, two possible scenarios arise: either $C_{il}$ and $C_{fr}$ intersect first or, $C_{ir}$ and $C_{fl}$ intersect first. In the first scenario, let $\overline{\beta}$ be the critical final orientation. $\mathcal{L}_{rsr} > \mathcal{L}_{rsl}$ if $\beta > \overline{\beta}$; $\mathcal{L}_{rsr} < \mathcal{L}_{rsl}$ if $\beta < \overline{\beta}$. In the second scenario, let $\overline{\alpha}$ be the critical initial orientation. $\mathcal{L}_{lsr} > \mathcal{L}_{rsr}$ if $\alpha > \overline{\alpha}$; $\mathcal{L}_{lsr} < \mathcal{L}_{rsr}$ if $\alpha < \overline{\alpha}$.

 Candidates \textit{LSR} and \textit{RSL} are infeasible if all the four circles intersect. Since $C_{fr}$ is on the left of $C_{ir}$, the lengths of the first and final segments of \textit{RSR} both exceed a length of $\pi$, thus \textit{RSR} cannot be the shortest path. Therefore, the shortest path can only be \textit{LRL}.
\end{proof}

\subsection{An example of orthogonal transformation}

By employing the orthogonal transformation detailed in Subsection \ref{subsec:Equivalency Group}, the optimal solutions for all 16 cases can be efficiently obtained. To illustrate, consider class $a_{44}$ as an example. Classes $a_{44}$ and $a_{11}$ belong to the same equivalency group. Let us denote the initial and final orientation angle pair for class $a_{11}$ as $(\alpha, \beta)$. The corresponding angle pair for class $a_{44}$ is $(-\alpha, -\beta)$. Under this condition, the topologically equivalent path for $(\alpha, \beta)$ is obtained by transforming the initial and final orientations to their conjugate forms.

For example, in class $a_{11}$, if the condition $C_{ir} \cap C_{fl} = \emptyset$ holds, the shortest path is designated as \textit{RSL}. Therefore, for class $a_{44}$, under the analogous condition $\bar{C}_{ir} \cap \bar{C}_{fl} = \emptyset$, the shortest path becomes \textit{LSR}. This result arises due to the symmetry and conjugacy between the two classes within the equivalency group. Table \ref{tab:a_44OptimalPath} summarizes the optimal path for class $a_{44}$.

\begin{table}[htbp]
    \caption{Shortest Path Corresponding to Class $a_{44}$}
    \label{tab:a_44OptimalPath}
    \renewcommand{\arraystretch}{1.3}
    \begin{center}
      \begin{tabular}{c|c|c|c} 
        \Xhline{1.2pt}
		\multicolumn{3}{c|}{Condition} & Shortest Path \\ 
		\hline
		\multicolumn{3}{c|}{$C_{il} \cap C_{fr} = \emptyset$} & RSL \\
		\hline
		\multirow{6}*{$C_{il} \cap C_{fr} \ne \emptyset$} & \multirow{3}*{$\alpha > \beta$} & $t_{lsl} < \pi, S_{44}^1 > 0$ & LSL \\
		~							 &~								& $t_{lsl} > \pi, S_{44}^2 > 0$ & RSL \\
		~							 &~								&     Otherwise					  & LRL \\
		\cline{2-4}
		~							 &\multirow{3}*{$\alpha < \beta$} & $q_{rsr} < \pi, S_{44}^3 > 0$ & RSR \\
		~							 &~								& $q_{rsr} > \pi, S_{44}^4 > 0$ & RSL \\
		~							 &~								&     Otherwise					  & RLR \\ 
		\Xhline{1.2pt}
	  \end{tabular}
    \end{center}
  \end{table}

  The switching functions in Table \ref{tab:a_44OptimalPath} are
  \begin{equation}
    \begin{aligned}
      S_{44}^1 &= 2(p_{lrl} - \pi) - p_{lsl}, \\
	  S_{44}^2 &= 2(t_{lrl} + q_{lrl}) - (p_{rsl} + 2q_{rsl}) + 2\pi, \\
	  S_{44}^3 &= 2(p_{rlr} - \pi) - p_{rsr}, \\
	  S_{44}^4 &= 2(t_{rlr} + q_{rlr}) - (p_{rsl} + 2t_{rsl}) + 2\pi.
    \end{aligned}
  \end{equation}

\section{Result}
\label{sec:Experimental Results}
The scheme of identifying the shortest Dubins path for the \textit{short path case} between two configurations is derived without explicitly calculating lengths of all candidates of the Dubins set. In this section, the proposed method is tested and compared with the traditional approach in terms of the number of segment lengths computed and the overall required time. The method was conducted in C++ and executed on an Intel Core i5-13500 CPU with 16 GB RAM.

We demonstrate how the proposed method can be employed to efficiently identify the shortest path. Given the initial and final positions and orientations, the method first determines whether the problem falls into the \textit{long path case} or the \textit{short path case}. If it belongs to the \textit{short path case}, the method further identifies the specific class to which it corresponds. Finally, the corresponding table is utilized to determine the type of the shortest curve. For instance, consider the initial and final configurations are $(0, 0, \pi / 4)$ and $(1, 0, 3\pi / 4)$, respectively. Based on (\ref{eq: d_threshold}), this case falls into the \textit{short path case}. Subsequently, according to the orientations, it is classified as $a_{14}$ class. By referencing Table \ref{tab:a_14OptimalPath}, the shortest path is determined to be of type \textit{LRL}.

To find the shortest candidate, the proposed method requires the explicit computation of only 3.1 path segment lengths on average. For configurations belonging to $\mathbb{E}_6$, a single segment length is sufficient to determine the shortest Dubins path pattern. In the worst-case scenario, no more than seven segment lengths need to be evaluated. In the contrast, the traditional approach evaluates all 18 ($3 \times 6$) segments. We generate $10^6$ random pairs of initial and final configuration queries. Using the proposed method, the shortest Dubins path pattern is identified in an average of 0.47 \textmu s, compared to 1.22 \textmu s with the conventional method. These results demonstrate that our approach significantly improves efficiency in identifying the shortest Dubins path pattern, especially in the case of numerous repeated computation.


\section{Conclusion}
\label{sec:CONCLUTIONS}
While computing the length of a single Dubins path involves relatively low computational cost, identifying the shortest path among all possible Dubins path types remains a nontrivial challenge. This paper proposes an improved method for determining the shortest path pattern within the Dubins set, particularly for cases involving two closely spaced configurations. The proposed approach simplifies the problem by systematically reducing the set of candidate paths that need to be evaluated. In contrast to conventional methods—which compute and compare the lengths of all possible path types—our method significantly reduces computational complexity. Given that many motion planning algorithms are built upon Dubins path formulations, the proposed technique has the potential to substantially enhance their overall efficiency.

Future work will focus on refining Dubins paths by incorporating smooth transitions to address practical limitations in motion execution. While Dubins paths ensure minimal travel distance, their discontinuous curvature can result in jerky movements, posing challenges for motion controllers tasked with accurate path tracking. To improve real-world applicability, we aim to develop methods that smooth the curvature while preserving near-optimal path length. Additionally, this research can be extended to time-optimal trajectory planning, wherein optimizing the velocity profile allows a vehicle to traverse the path in the shortest possible time.

\bibliographystyle{IEEEtran}
\bibliography{reference}

\begin{thebibliography}{10}
\providecommand{\url}[1]{#1}
\csname url@samestyle\endcsname
\providecommand{\newblock}{\relax}
\providecommand{\bibinfo}[2]{#2}
\providecommand{\BIBentrySTDinterwordspacing}{\spaceskip=0pt\relax}
\providecommand{\BIBentryALTinterwordstretchfactor}{4}
\providecommand{\BIBentryALTinterwordspacing}{\spaceskip=\fontdimen2\font plus
\BIBentryALTinterwordstretchfactor\fontdimen3\font minus
  \fontdimen4\font\relax}
\providecommand{\BIBforeignlanguage}[2]{{%
\expandafter\ifx\csname l@#1\endcsname\relax
\typeout{** WARNING: IEEEtran.bst: No hyphenation pattern has been}%
\typeout{** loaded for the language `#1'. Using the pattern for}%
\typeout{** the default language instead.}%
\else
\language=\csname l@#1\endcsname
\fi
#2}}
\providecommand{\BIBdecl}{\relax}
\BIBdecl

\bibitem{gonzalez2015review}
D.~Gonz{\'a}lez, J.~P{\'e}rez, V.~Milan{\'e}s, and F.~Nashashibi, ``A review of
  motion planning techniques for automated vehicles,'' \emph{IEEE Transactions
  on Intelligent Transportation Systems}, vol.~17, no.~4, pp. 1135--1145, 2015.

\bibitem{dubins1957curves}
L.~E. Dubins, ``On curves of minimal length with a constraint on average
  curvature, and with prescribed initial and terminal positions and tangents,''
  \emph{American Journal of Mathematics}, vol.~79, no.~3, pp. 497--516, 1957.

\bibitem{oliveira2018trajectory}
R.~Oliveira, P.~F. Lima, M.~Cirillo, J.~M{\aa}rtensson, and B.~Wahlberg,
  ``Trajectory generation using sharpness continuous {D}ubins-like paths with
  applications in control of heavy-duty vehicles,'' in \emph{2018 European
  Control Conference (ECC)}.\hskip 1em plus 0.5em minus 0.4em\relax IEEE, 2018,
  pp. 935--940.

\bibitem{tian2021continuous}
Y.~Tian, Z.~Chen, C.~Xue, Y.~Sun, and B.~Liang, ``Continuous curvature turns
  based method for least maximum curvature path generation of autonomous
  vehicle,'' in \emph{IECON 2021--47th Annual Conference of the IEEE Industrial
  Electronics Society}.\hskip 1em plus 0.5em minus 0.4em\relax IEEE, 2021, pp.
  1--6.

\bibitem{schildbach2023continuous}
G.~Schildbach, ``A continuous collision detection algorithm for {D}ubins
  paths,'' in \emph{2023 IEEE Intelligent Vehicles Symposium (IV)}.\hskip 1em
  plus 0.5em minus 0.4em\relax IEEE, 2023, pp. 1--6.

\bibitem{pharpatara20153d}
P.~Pharpatara, B.~H{\'e}riss{\'e}, and Y.~Bestaoui, ``3{D}-shortest paths for a
  hypersonic glider in a heterogeneous environment,'' \emph{IFAC-PapersOnLine},
  vol.~48, no.~9, pp. 186--191, 2015.

\bibitem{park2022three}
S.~Park, ``Three-dimensional {D}ubins-path-guided continuous curvature path
  smoothing,'' \emph{Applied Sciences}, vol.~12, no.~22, p. 11336, 2022.

\bibitem{beard2012small}
R.~W. Beard and T.~W. McLain, \emph{Small unmanned aircraft: Theory and
  practice}.\hskip 1em plus 0.5em minus 0.4em\relax Princeton university press,
  2012.

\bibitem{oettershagen2017towards}
P.~Oettershagen, F.~Achermann, B.~M{\"u}ller, D.~Schneider, and R.~Siegwart,
  ``Towards fully environment-aware {UAV}s: Real-time path planning with online
  3{D} wind field prediction in complex terrain,'' \emph{arXiv preprint
  arXiv:1712.03608}, 2017.

\bibitem{shkel2001classification}
A.~M. Shkel and V.~Lumelsky, ``Classification of the {D}ubins set,''
  \emph{Robotics and Autonomous Systems}, vol.~34, no.~4, pp. 179--202, 2001.

\bibitem{cho2006efficient}
G.~Cho and J.~Ryeu, ``An efficient method to find a shortest path for a
  car-like robot,'' in \emph{Intelligent Control and Automation: International
  Conference on Intelligent Computing, ICIC 2006 Kunming, China, August 16--19,
  2006}.\hskip 1em plus 0.5em minus 0.4em\relax Springer, 2006, pp. 1000--1011.

\bibitem{sadeghi2016efficient}
A.~Sadeghi and S.~L. Smith, ``On efficient computation of shortest {D}ubins
  paths through three consecutive points,'' in \emph{2016 IEEE 55th Conference
  on Decision and Control (CDC)}.\hskip 1em plus 0.5em minus 0.4em\relax IEEE,
  2016, pp. 6010--6015.

\bibitem{ji20213d}
C.~Ji, C.~Wang, M.~Song, and F.~Wang, ``A 3{D} {D}ubins curve constructing
  method based on particle swarm optimization,'' in \emph{International
  Conference on Parallel and Distributed Computing: Applications and
  Technologies}.\hskip 1em plus 0.5em minus 0.4em\relax Springer, 2021, pp.
  150--160.

\bibitem{nayak2023heuristics}
A.~Nayak and S.~Rathinam, ``Heuristics and learning models for {D}ubins minmax
  traveling salesman problem,'' \emph{Sensors}, vol.~23, no.~14, p. 6432, 2023.

\bibitem{lim2023circling}
J.~Lim, F.~Achermann, R.~Girod, N.~Lawrance, and R.~Siegwart, ``Circling back:
  {D}ubins set classification revisited,'' in \emph{Energy Efficient Aerial
  Robotic Systems Workshop}.\hskip 1em plus 0.5em minus 0.4em\relax ETH Zurich,
  2023.

\end{thebibliography}
\end{document}